\documentclass[]{interact}
\usepackage{epstopdf}
\usepackage[caption=false]{subfig}
\usepackage{float}
\usepackage{mathrsfs}
\usepackage{textcomp}

\usepackage[numbers,sort&compress]{natbib}
\bibpunct[, ]{[}{]}{,}{n}{,}{,}
\renewcommand\bibfont{\fontsize{10}{12}\selectfont}
\makeatletter
\def\NAT@def@citea{\def\@citea{\NAT@separator}}
\makeatother

\newtheorem{theorem}{Theorem}[section]
\newtheorem{Example}{Example}[section]
\newtheorem{Lemma}{Lemma}[section]

\begin{document}


\title{Integro-differential equations in angular stabilization of drone's motion by distributed feedback control}



\author{
\name{Alexander Domoshnitsky${}^2$, Oleg Kupervasser${}^{1,2}$ \thanks{CONTACT Oleg Kupervasser Email: olegkup@yahoo.com} and Anatoly Polonsky${}^2$}
\affil{${}^1$  Sami Shamoon College of Engineering, South District, Be'er Sheva, Israel } \affil{ ${}^2$ Ariel University, Ari'el, ISRAEL}   }


\maketitle

\begin{abstract}
In this paper, we propose angular stabilization of drone's
motion using distributed feedback control in the form of an
integral operator. It should be stressed that the memory of this
integral operator could be unbounded. It is intuitively clear that
large length of the observation time open new possibilities to
construct better control based on previous states of the control
object. Unbounded memory in control requires the creation of a
certain approach different from standard ones to the study of
integro-differential equations. One of the goals of this article
is to propose a certain universal approach that allows us to study
the stability of integro-differential equations in the case of
unbounded memory in the integral operator specifying the feedback
control in stabilization. The approach we propose allows us to
reduce the study of integro-differential equations to the analysis
of systems of ordinary differential equations. In general, such
systems can consist of an infinite number of equations. In relation to
the so-called linear approximation in the problem of angle
stabilization manages to limit itself to relatively simple exponential kernels
in the integral control and arrive at a system with a finite
number of equations. The examples explain that more complex
kernels, for example, linear combinations of the exponential kernels, can enhance the stabilization capabilities. We obtain new
unexpectable results on the exponential stability of
integro-differential equations. Then we apply them to
stabilization of drone's flight.

AMS classification 45M10, 34K20, 45D05
\end{abstract}


\section{Introduction}

Integro-differential equations are used as mathematical models in
wide class of technological, biological and medical processes.
Evolutionary integral equations and their applications are the
object of the book \cite{Pruss}. Theory of heat conduction with
finite wave speeds are proposed in \cite{Gurtin}. Phase-field
equations with memory was studied in \cite{Novick,Rotstein}.
Viscoelastic properties of polymers with the use of
integro-differential equations as models were studied in
\cite{Ferry}. \ Mathematical problems in viscoelasticity with
integro-differential equations were formulated and studied in
\cite{Fabrizio,Renardy}. Boundary value problems in linear
viscoelasticity were considered in \cite{Golden}. Stability of
integro-differential equations in viscoelasticity was studied in \cite%
{Drozd1}. Important applications of integro-differential equations
we can
see in epidemiological models \cite%
{Aguiar2008,Chen,Guan,Driessche,Huang,Reich,Vanessa,Xu2017}.
Various applications present one of the main reasons in
development of the theory of these equations. Let us emphasize
here that integro-differential equations represent a complicated
object for theoretical analysis and, most importantly, for
numerical solution methods since integral terms accumulate
calculation errors. The larger memory length, the larger these
errors will be. The idea of somehow reducing the analysis of
integro-differential equations to the analysis of certain averaged
ordinary differential equations that can be approximately solved
with known numerical methods has existed for a long time. The
averaging method for solving integro-differential equations was
proposed in the book \cite{Filatov}. The difficulty, and in many
cases the impossibility, of estimating the error when moving on to
considering the averaged equation poses a serious problem for
researchers. This difficulty leads to creation of new stability
methods and estimates of solutions to integro-differential
equations. Note several books on the theory of
integro-differential equations
\cite{Burton,Corduneanu,cord2,Gripenberg}. Their generalizations
in the form of equations with Volterra operators was proposed in \cite%
{AMR,AP}. Stability of integro-differential equations was studied in \cite%
{BDK,drozd2,drozd,Gil,Hara}. Versions of the Floquet theory for
integro-differential equations are constructed in \cite%
{AgarwalLup,Becker,Belbas}.

We try to propose a corresponding approach to the study of
integro-differential equations. Under natural assumptions,
integro-differential equations can be reduced to systems of
ordinary differential equations. In a general case these systems
can consist of infinite number of differential equations. In
modern real-world applications the kernels are simple enough; for example, finite linear combinations of exponential kernels, and
we come to finite-dimensional systems. Existing methods for their
analysis and numerical solving can be then used. 

Our main motivation lies in the task of angular stabilization of drone's motion. Drone
flight control tasks have recently become key to mathematical modelling in industrial
and military aviation technology. The aerodynamic aspects of the flight of aircraft and
similarly-shaped flying machines have been developed for quite some time and are included
in the basic textbooks (see, for example, the book \cite{Bodner1} from 1961).
With the development of technology, unmanned aerial vehicles (UAVs) have become a
very important area of research around the world in recent decades (see, for example, the
state of the art in the papers \cite{Song,Yang,Yajing,Abdillah}, where models of quadrotors are considered
and many important references are given. Fixed-wing drones were also developed \cite{Triputra}).
Modeling flight of such drone motion in the form of system of delay differential equations
is described in the paper \cite{Avasker1}.
Note our projects in this area, where
drone's flight is studied (these projects are made using the grant \cite%
{DomKup2017} from China private company Hangzhou Avisi Electronics
Co., Ltd and grants NOFAR \cite{DomKup2020} and KAMIN
\cite{DomKup2018} from The Israel Innovation Authority, Ministry
of Economics and Trade of State of Israel). Modelling flight
of such drone motion in the form of system of delay differential
equations can be found in the paper \cite{Avasker1}. In our paper, we use the distributed feedback
control in the form of integral term to stabilize the flight of
drones. We will focus our attention on a problem of angular
stabilization for drone motion in vertical plane. It is
intuitively clear that a long time observation which can be
described by the feedback control in the form of integral term
with long memory allows researchers to construct better control
than the use of previous state at one point only. This opens new
possibilities for stabilization. We obtain some unexpectable
results in this direction on the basis of our approach to
stability studies of integro-differential equations. 

Why does control in the form of integral term look naturally? The answer on this question is very simple.
First of all, noise can imply measurements. Engineers to prevent this noise prefer to use average value of many measurements.
This is a way to integral average value which is presented very widely in Internet (see, integral control). But in construction
of the control of drone's flight, engineers use measurements of different devices situating in different place. Distance between them 
can be kilometers. Delay in coming signals from different devices and probability of errors in measurements of each of them lead to the use of
complicated kernels in the integral control. There were only a few papers to note (see \cite{Yajing} and references there) connected to drone's flight.
The main goal of our paper is to fill this gap and to propose convenient tests of stabilization of angular drone's motion by integral delay feedback control.     

The idea of our approach to studying stability is that an auxiliary system of
a larger dimension than the original system is built. Solution
vector of the original system of integro-differential equations
and corresponding components of the solution vector for this built
system of ordinary differential equations coincide. This new
system, despite its larger dimension, turns out to be much simpler
for stability analysis on the basis of existing methods. We
explain how to build this auxiliary system. The fact that memory
can be unbounded, unlike known approaches to studying the
stability of integro-differential equations, that require a small
memory length, is not an obstacle in our analysis, but on the
contrary helps in stability studies. We also emphasize that we
will not need to store all previous states of flight of our drone
in the memory of the control device since the extended system
consists of ordinary differential equations and it
\textquotedblleft absorbs\textquotedblright\ all previous states
of drone's flight in an implicit form.

Here, it is important for us to stress the following fact.
Theoretical research around the design of automatic observers and
controllers using so-called sampled data has become an important
topic in the last 20 years (see, for example, its description in
\cite{Fridman1}). Note that the idea of sampled data was used in
the book \cite{Fridman2}. It was demonstrated in the paper
\cite{Cacace} that appearing of time-varying gains for a specific
class of observers can be efficient. Usually, the switch to
sampled data is due to the fact that the control only uses the
values of the solution at some discrete points. That is, the range
of possible $x(t)$ values is narrowed to several discrete states
$x(t_{1}),\ldots ,x(t_{m})$. When we use an integral operator
$\int \limits_{0}^{t}k(t,s)x(s)ds$\ in the feedback control, we
seem to be going in the opposite direction, extending possible
values of $x(s)$ to the continuum on which the control effect is
built. It seems paradoxical that, in spite of this, but our
reduction of integro-differential systems to systems of ordinary
differential equations actually \textquotedblleft
narrows\textquotedblright\ the data. In our future research, we
will look at an explicit combination of integral feedback control
and sampled data. This can lead to completely new tests of
stability and methods of stabilization through distributed
feedback control.

Our paper is built as follows. In Section 2 we construct a model
of drone's flight. As well as we know, it is the first time in the
the mathematical literature that the feedback control in the form
of integral terms with exponential kernels is used. In Section 3, we describe our method
for analysis of integro-differential equations. In Section 4, we
obtain zones of stability, which depend on parameters of the
control for the study of integro-differential equations. In
Section 5, we demonstrate how our results can work in
stabilization of drone's flight.

\bigskip

\section{Construction of model of drone's flight}

\bigskip

Let's consider a description of devices for determining the
angular position of a drone during longitudinal movement in a
vertical plane perpendicular to a ground sea level surface,
derivation of equations of angular motion, stabilization of the
angular position of a drone using the theory described in the
previous parts, as well as finding control parameters that ensure
this stabilization. For these purposes we use works
\cite{Avasker1, Bodner1, Abahre1, Madgwick1, Madgwick2, Kotaro1,
Indicator1}.

The angular position of the drone during longitudinal movement in
the vertical plane of a perpendicular to ground sea level surface
is characterized by two angles: $\hat{a}(\bar{t})$ - angle of
attack, i.e. angle between longitudinal axis of a drone and
projection of drone velocity on the symmetry plane of the drone,
$\vartheta (\bar{t})$ - pitch angle,
i.e., angle between longitudinal drone axis and horizontal plane, here $\bar{%
t}$ describes time.

Let us look at the instruments measuring these angles. The pitch angle $%
\vartheta (\bar{t})$ is measured by integrating the readings of
MEMS gyroscopes, which measure the angular velocity of rotation.
Accumulating errors at a constant speed of movement are corrected
by measuring deviations from the vertical gravitational field,
determined by the gravitational acceleration g. The direction and
magnitude of g relative to the drone is measured using MEMS
accelerometers \cite{Abahre1, Madgwick1, Madgwick2}. The angle of
attack $\hat{a}(\bar{t})$ is measured using indicated airspeed
anemometers \cite{Kotaro1, Indicator1}. These two devices make
measurements very quickly and have almost zero time delay.

The flight of the drone is characterized by the following steady
state parameters: $\alpha _{0}$ - angle of attack, i.e., angle
between longitudinal axis of a drone and projection of drone
velocity on the symmetry plane of the drone; $\vartheta _{0}$ -
pitch angle, i.e., angle between longitudinal drone axis and
horizontal plane; $V_{0}$ - flight velocity tangent to trajectory
(with respect of air); $H_{0}$ - height above mean sea level of a
drone flight; ${(U_{y})}_{0}$ - wind velocities along the axis
$y$.

Let us consider perturbation of these steady state parameters: $V(\bar{t}%
)=V_{0}+\Delta V(\bar{t}),$ $\hat{\alpha}(\bar{t})=\alpha
_{0}+\Delta \alpha
(\bar{t}),$ $\hat{\vartheta}(\bar{t})=\vartheta _{0}+\Delta \vartheta (\bar{t%
}),$ $U_{y}(\bar{t})={(U_{y})}_{0}+\Delta U_{y}(\bar{t}),$ $H(\bar{t}%
)=H_{0}+\Delta H(\bar{t}).$ Their normalized values are usually
taken as

$$
\left\{
\begin{array}{c}
v(t)=\Delta V(\tau _{a}t)/V_{0}; \\
h(t)=\Delta H(\tau _{a}t)/V_{0}\tau _{a}; \\
\alpha (t)=\Delta \alpha (\tau _{a}t); \\
\vartheta (t)=\Delta \vartheta (\tau _{a}t); \\
\nu _{y}(t)=\Delta U_{y}(\tau _{a}t)/V_{0};%
\end{array}%
\right. .\eqno(2.1)
$$
\begin{figure}[H]
\centering
\includegraphics[width=10cm]{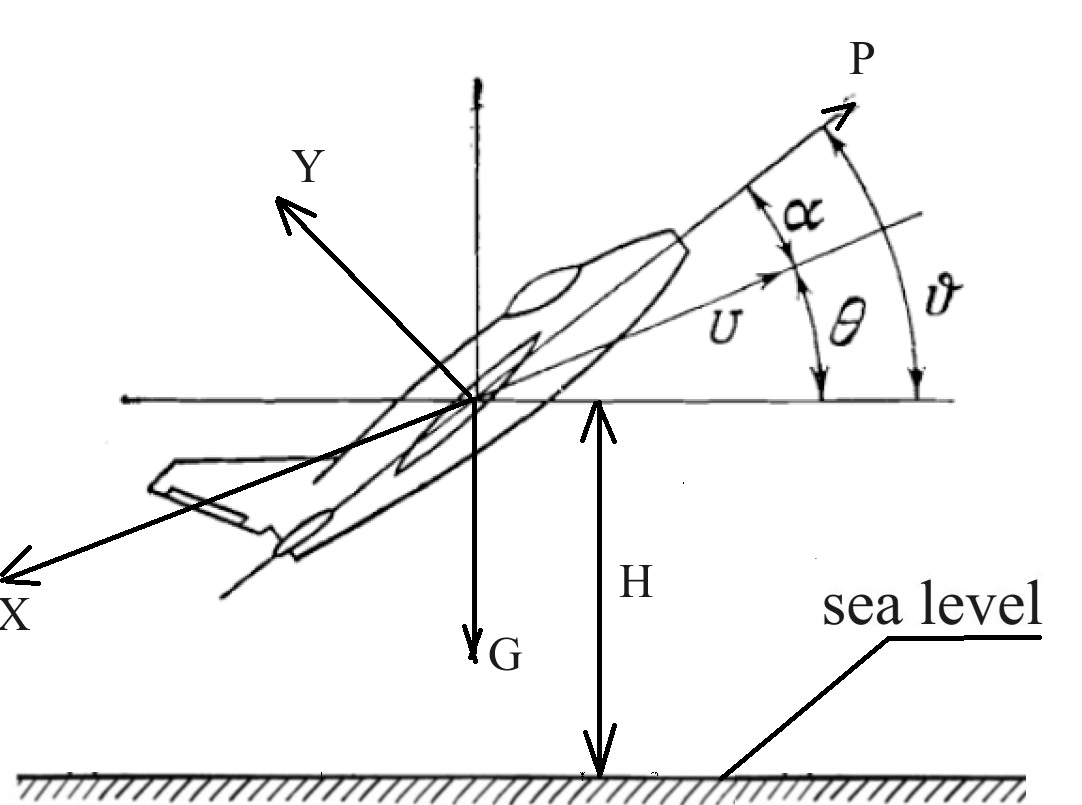}
\caption{Parameters of drone's longitudinal motion, Y - carrying force orthogonal to flight velocity, X - resistance force opposite to V, G - gravitation force, P - tractive force directed along longitudinal drone axis, V - flight velocity tangent to trajectory (with respect to air), H - height above mean sea level of a drone flight, $\vartheta$ -pitch angle, i.e., angle between longitudinal drone axis and horizontal plane, $\theta$ - tilting of velocity about horizontal plane, $\alpha$ - angle of attack, i.e., angle between longitudinal. Full description of these parameters can be found in the paper \cite{Avasker1}}
\end{figure}




where ${t}={\frac{\bar{t}}{{\tau _{a}}}}$; $\tau _{a}=m/(\rho _{0}V_{0}S)$; $%
m$ - drone's mass; $S$ - area of wings; $\rho (H)$ - air density
at a flight height $H$; $\rho (H)=\rho
(0)(T_{H}(H)/T(0))^{1/(\gamma -1)}$;$\rho _{0}=\rho
(0)(T_{H}(H_{0})/T(0))^{1/(\gamma -1)}$; $T_{H}(H)$ - temperature
at a flight height; $T_{H}(H)=T(0)-\beta H$; $T(0)$, $\rho (0)$ -
temperature and air density at mean sea level; $\gamma $- adiabatic constant;%
$R$ -gas constant; $\beta $-temperature gradient over height.

Our basic model consists of four equations for normalized
perturbations of
four drone motion parameters with respect to four steady state values \cite%
{Avasker1, Bodner1}:

$$
\left\{
\begin{array}{c}
v^{\prime }(t)+n_{11}v(t)+n_{12}\alpha (t)+n_{13}\vartheta
(t)+n_{14}h(t)=n_{p}\delta _{p}(t)+f_{1}(t) \\
\alpha ^{\prime }(t)-n_{21}v(t)+n_{22}\alpha (t)-\vartheta
^{\prime
}(t)-n_{23}\vartheta (t)+n_{24}h(t)=f_{2}(t) \\
n_{31}v(t)+n_{0}\alpha ^{\prime }(t)+n_{32}\alpha (t)+\vartheta
^{\prime \prime }(t)+n_{33}\vartheta ^{\prime
}(t)+n_{34}h(t)=-n_{B}\delta
_{B}(t)+f_{3}(t) \\
h^{\prime }(t)-n_{41}v(t)+n_{42}\alpha (t)-n_{42}\vartheta (t)=\nu _{y}(t)%
\end{array}%
\right. .\eqno(2.2)
$$

In (2.2) we have controlling signals: $\delta _{p}(t)$-position of
drone central control knob, $\delta _{B}(t)$ - deviation of drone
control elevator; random external parameters:
$f_{1}(t),f_{2}(t),f_{3}(t)$ - random forces and random moments of
forces, $\nu _{y}(t)$ - random fluctuation of
vertical component of wind velocities along axes $y$. The random functions $f_{1}(t),f_{2}(t),f_{3}(t),\nu _{y}(t) $ are bounded functions with bounded derivatives with respect to time. The coefficients $%
n_{ij} (i,j=1,2,3,4),n_{B},n_{p},n_{0}$ are constant parameters
obtaining by linearization of basic nonlinear equations, they can
depend on parameters of steady state motion. The values of these
constant parameters can be found in the Table 1 in
\cite{Avasker1}.

Partial simplified case of the system (2.2) can be obtained under
the following assumptions \cite{Bodner1}: We suppose that dependence on a flight height is negligible: $n_{14}=0$; $%
n_{24}=0$; $n_{34}=0$. Velocity changes are much slower than angle
changes so $v=0$. Horizontal motion so tilting of velocity about
horizontal plane is equal to
zero \cite{Bodner1}: $\theta_{0}=0 \Rightarrow n_{23}={\frac{{2 mg}}{{%
\varrho_{0} S V_{0}^{2}}}} sin(\theta_{0}) =0$, where $\theta_{0}=%
\vartheta_{0}-\alpha_{0}$ - steady state value of tilting velocity
about horizontal plane. $g$ - gravitational acceleration of a free
fall.

This leads us to the following system:%
\[
\left\{
\begin{array}{c}
\alpha^{\prime }(t)+n_{22}\alpha(t)-\vartheta ^{\prime}(t)=f_{2}(t) \\
n_{0}\alpha^{\prime
}(t)+n_{32}\alpha(t)+\vartheta^{\prime\prime}(t)+n_{33}\vartheta^{%
\prime}(t)=-n_{B}\delta_{B}(t)+f_{3}(t)%
\end{array}%
\right .\Leftrightarrow
\]

\[
\left\{
\begin{array}{c}
(\alpha(t)-\vartheta(t))^{\prime }+n_{22}\alpha(t)=f_{2}(t) \\
n_{0}\alpha^{\prime }(t)+n_{32}\alpha(t)+(\alpha^{\prime
}(t)+n_{22}\alpha(t)-f_{2}(t))^{\prime}+n_{33}(\alpha^{\prime
}(t)+n_{22}\alpha(t)-f_{2}(t)) \\
=-n_{B}\delta _{B}(t)+f_{3}(t)%
\end{array}%
\right .\Leftrightarrow
\]

$$
\left\{
\begin{array}{c}
\theta^{\prime }(t)=n_{22}\alpha(t)-f_{2}(t) \\
\alpha^{\prime\prime}(t)+\gamma_{1}\alpha
^{\prime}(t)+\gamma_{2}\alpha(t)=-n_{B}\delta
_{B}(t)+f_{3}(t)+f_{2}^{\prime}(t)+n_{33}f_{2}(t)%
\end{array}%
\right. ,\eqno(2.3)
$$
where $\theta(t)= \vartheta(t)-\alpha(t)$,
$\gamma_{1}=n_{0}+n_{33}+n_{22}$,
$\gamma_{2}=n_{32}+n_{33}n_{22}$.

Constants $n_{0},n_{33}$ are defined from aerodynamic damping moment: $%
n_{0}=-{\frac{\mu }{\tau _{a}}}({\frac{\partial m_{2z}}{\partial \dot{\alpha}%
}})_{0}=-{\frac{\mu }{\tau _{a}}}k^{\prime }k$;
$n_{33}=-{\frac{\mu }{\tau
_{a}}}({\frac{\partial m_{2z}}{\partial \dot{\vartheta}}})_{0}=-{\frac{\mu }{%
\tau _{a}}}{\frac{L_{1}}{M(V_{0},H_{0})}}k$. $()_{0}$- This symbol
means
that a value inside of the brackets is taken for the steady values $V_{0}$, $%
\theta _{0}$, $\alpha _{0}$, $\vartheta _{0}$, $H_{0}$, ${(U_{y})}_{0}$, $%
\delta _{p}=0$, $\delta _{B}=0$, $f_{1}(t)=0$, $f_{2}(t)=0$,
$f_{3}(t)=0$.

Here $M_{z}$ - total moment of aerodynamical forces with respect
of transversal axis,

$m_{z}$- coefficient of moment defined by the formula:

$m_{z}(\alpha ,M(V,H),\dot{\alpha},\dot{\vartheta},\delta
_{B})\doteq{\frac{M_{z}}{b_{a}S\rho (H)V^{2}/2}}$.

Coefficient of moment $m_{z}$ consists of two components $m_{1z}$,
$m_{2z}$:

$m_{z}(\alpha ,M(V,H),\dot{\alpha},\dot{\vartheta}%
,\delta _{B})=m_{1z}(\alpha ,M(V,H),\delta _{B})+m_{2z}(M(V,H),\dot{\alpha},%
\dot{\vartheta}),$ where

$m_{2z}(M(V,H),\dot{\alpha},%
\dot{\vartheta})=k({L_1 \over M(V,H)}\dot{\vartheta}+k^{\prime
}\dot{\alpha})$,
 $L_{1}$-distance
from tail unit to center of mass, $k$, $k^{\prime }$- constants,

$\mu ={%
\frac{b_{a}m}{2r_{z}^{2}\rho _{0}S}}$, $b_{a}$- length of wing chord, $%
r_{z}^{2}={\frac{J_{z}}{m}}$, $r_{z}$-inertial radius, $J_{z}$-
inertial moment of drone with respect of axis $z$, $M(V,H)\doteq
{\frac{V}{a(H)}}>0$
- Mach number, $a(H)=\sqrt{\gamma RT_{H}(H)}$ - sound velocity. Constants $%
n_{32}$, $n_{22}$ are defined from aerodynamic damping moment: $n_{32}=-\mu (%
{\frac{\partial m_{1z}}{\partial \alpha }})_{0}$, $n_{22}={\frac{1}{2}}(({%
\frac{\partial c_{y}}{\partial \alpha }})_{0}+(c_{x})_{0})$.
Denote Y- carrying force orthogonal to flight velocity, X-
resistance force opposite to V, $c_{x}(\alpha
,M(V,H))={\frac{X}{S\rho (H)V^{2}/2}}$, $c_{y}(\alpha
,M(V,H))={\frac{Y}{S\rho (H)V^{2}/2}}$.

For stable jet, we have: $L_{1}>0$; $k<0$; $k^{\prime }>0$; $\mu
>0$ so constants $n_{0}$, $n_{33}$ are more than zero: $n_{0}>0$;
$n_{33}>0$,

$(c_{x})_{0}>0$; $({\frac{\partial c_{y}}{\partial \alpha }})_{0}>0$ so $%
n_{22}={\frac{1}{2}}(({\frac{\partial c_{y}}{\partial \alpha }}%
)_{0}+(c_{x})_{0})>0$. As a result, we obtain $\gamma
_{1}=n_{0}+n_{33}+n_{22}>0.$

Jet aircraft has a small damping moment, so
$n_{33}n_{22}<<|n_{32}|.$As a result, the first term $n_{32}$ in
equation $\gamma _{2}=n_{32}+n_{33}n_{22}$ is leading term and
defines its sign. $\mu >0$ so the sign of $n_{32}$ depends on the
sign of $({\frac{\partial m_{1z}}{\partial \alpha }})_{0}$.

For modern drones for up to sound speeds $M<1.5,$ we have
$({\frac{\partial
m_{1z}}{\partial \alpha }})_{0}<0$. However, for $M\geq 1.5,$ we have $({%
\frac{\partial m_{1z}}{\partial \alpha }})_{0}\geq 0$. So we obtain $%
n_{32}>0 $ for $M<1.5$ and $n_{32}\leq 0$ for $M\geq 1.5$. Because
$n_{32}$
is leading term for $\gamma _{2}$ then $\gamma _{2}>0$ for $M<1.5$ and $%
\gamma _{2}\leq 0$ for $M\geq 1.5$.

\bigskip

\section{Method for studies of integro-differential equations}

\bigskip In this section, we propose an approach to studies of
integro-differential equations. Let us consider the equation
$$
x^{(n)}(t)=f(t,x(t),\ldots
,x^{(n-1)}(t),\int_{0}^{t}k^{0}(t,s)x(s)ds,\ldots
,\int_{0}^{t}k^{n-1}(t,s)x^{(n-1)}(s)ds),\eqno(3.1)
$$%
under the standard assumptions that $f$ is a Caratheodori function
and $k^{m}(t,s)$ is a square integrable function. Then in the
Hilbert space $~L_{2}~$ it can be represented as
$$
k^{m}(t,s)=\sum\limits_{i,j=1}^{\infty }c_{ij}^{m}\alpha _{i}(t)\alpha
_{j}(s),\text{ }m=0,\ldots ,n-1,\eqno(3.2)
$$%
where $~\alpha _{i}(t)~(i=1,2,\cdots )~$ is an orthogonal basis in $~L_{2},~$
and $~\sum\limits_{i,j=1}^{\infty }\left( c_{ij}^{m}\right) <\infty ~$ (see
\cite{Kolmogorov},Chapter 7 or \cite{Shilo}). This kernel can be rewritten
as
$$
k^{m}(t,s)=\sum\limits_{i,j=1}^{\infty }\left( c_{ij}^{m}\right) \beta
_{i}^{m}(t)e^{-\gamma _{ij}^{m}t}e^{\gamma _{ij}^{m}s}\beta _{j}^{\ast m}(s),%
\eqno(3.3)
$$%
where $\beta _{i}^{m}(t)=\alpha _{i}(t)e^{\gamma _{ij}^{m}t},$ $\beta
_{j}^{\ast m}(s)=\alpha _{j}(s)e^{-\gamma _{ij}^{m}s}.$ The function $%
e^{-\gamma _{ij}^{m}t}e^{\gamma _{ij}s}=e^{-\gamma _{ij}^{m}(t-s)}$ is the
Cauchy function of the equation
$$
\frac{dv_{ij}^{m}}{dt}(t)+\gamma _{ij}^{m}v_{ij}^{m}(t)=z(t),\eqno(3.4)
$$%
thus, its general solution can represented as
$$
v_{ij}^{m}(t)=\int_{0}^{t}e^{-\gamma
_{ij}^{m}(t-s)}z(s)ds+v_{ij}(0)e^{-\gamma _{ij}^{m}t}.\eqno(3.5)
$$%
Consider now the system
$$
x^{(n)}(t)=f(t,x(t),\ldots ,x^{(n-1)}(t),\sum\limits_{i,j=1}^{\infty }\left(
c_{ij}^{0}\right) \beta _{i}^{0}(t)v_{ij}^{0}(t),\ldots
,\sum\limits_{i,j=1}^{\infty }\left( c_{ij}^{n-1}\right) \beta
_{i}^{n-1}(t)v_{ij}^{n-1}(t)),\eqno(3.6)
$$%
$$
\frac{dv_{ij}^{m}}{dt}(t)+\gamma _{ij}v_{ij}^{m}(t)=\beta _{j}^{\ast m}(t) x(t),%
\text{ }m=0,\ldots ,n-1,\text{ }i,j=1,2,3,\ldots \eqno(3.7)
$$%
where the components $v_{ij}^{m}(t)$ satisfy the initial conditions
$$
v_{ij}^{m}(0)=0,m=0,\ldots ,n-1,\text{ }i,j=1,2,3,\ldots \eqno(3.8)
$$

We obtain the following

{\bf Theorem 3.1. }{\it The solution }$x(t)${\it \ of system (3.1) and the
first component of the solution-vector }$\\$
$
(x(t),v_{11}^{0}(t),v_{12}^{0}(t),v_{21}^{0}(t),v_{13}^{0}(t),v_{22}^{0}(t),v_{31}^{0}(t),\ldots,$

$
v_{11}^{n-1}(t),v_{12}^{n-1}(t),v_{21}^{n-1}(t),v_{13}^{n-1}(t),v_{22}^{n-1}(t),v_{31}^{n-1}(t),\ldots )$
{\it \ \\of system (3.6),(3.7), satisfying the initial condition (3.8) for
the components }$v_{ij}^{m}(t),${\it \ coincide.}

{\bf Remark 2.1.} Theorem 3.1 reduces analysis of
integro-differential system (3.1) to the studies of
infinite-dimensional system (3.6),(3.7). Analysis of stability of
such systems were started, as well as we know, by K.P.Persidskii
\cite{Pers1,Pers2,Pers3}. Existence and uniqueness results were
proposed by R.Bellman \cite{Bellman}. In the book \cite{Valeev}, a
theory of infinite-dimensional systems was presented.
Corresponding \textquotedblleft smalness\textquotedblright\ of the
sum of the coefficients were assumed. Questions of representation
of solutions and existence and uniqueness results were proposed in
\cite{Lena1,Lena2}, in the case of only finite number of
components in every equation (This property was called as a
condition $V$). Positivity of solutions to boundary value problems
for integro-differential equations were studied in
\cite{DomGolInfinite}, were the condition $V$ was avoided. This
could allow us to get new tests of stability for
integro-differential systems on the basis of the theory proposed
in \cite{AMR}. It should be noted that we have infinite number of
the components in the equation (3.6) and the condition $V$ is not
fulfilled.

{\bf Remark 3}.{\bf 2}. Elementary examples of
integro-differential equations with simple kernels for which basic
operations lead to system of ordinary differential equations were
discussed in the book \cite{Burton} and the paper \cite{PSM}.

{\bf Remark 3.3.} If we limit ourselves by the finite number \ in
the sum,
i.e.,%
$$
k^{m}(t,s)=\sum\limits_{i,j=1}^{q}c_{ij}^{m}\alpha _{i}(t)\alpha _{j}(s),%
\text{ }m=0,\ldots ,n-1,\eqno(3.9)
$$%
we obtain the integro-differential system in which there are only
finite number of equations in (3.7). All classical methods for
analysis of this system can be used. Researchers have difficulties
in the use of developed numerical methods for solving
integro-differential equations. The reason of this is in the fact
that integral terms accumulate numerical errors. Our reducing
integro-differential equation to system of ordinary differential
equations gives an approach to solve this problem.

{\bf Remark 3.4.} Increasing $q$ improves our possibility in stabilization.
Actually, equation $x^{\prime \prime }(t)=0$ is unstable. Adding integral
term does not allow to achieve the exponential stability\ for the equation
$$
x^{\prime \prime }(t)+\beta _{1}\int_{0}^{t}e^{-\gamma
_{1}(t-s)}x(s)ds=0,\eqno(3.10)
$$%
since its characteristic equation (obtained through reducing by
Theorem 3.1 to system of ordinary differential equations) is
$\lambda ^{3}+\gamma _{1}\lambda ^{2}+$ $\beta _{1}=0$, and we see
that  the necessary condition of the exponential stability given
by the Hurwitz criteria is not fulfilled. If we consider the
equation
$$
x^{\prime \prime }(t)+\beta _{1}\int_{0}^{t}e^{-\gamma
_{1}(t-s)}x(s)ds+\beta _{2}\int_{0}^{t}e^{-\gamma _{2}(t-s)}x(s)ds=0,%
\eqno(3.11)
$$%
we can get the exponential stability. For example, let us take the
following parameters $\beta _{1}=50,$ $\beta _{2}=-28,$ $\gamma
_{1}=7,$ $\gamma _{2}=4.$ Theorem 3.1 reduces analysis of this
integro-differential equation to system of ordinary differential
equations with the following roots of the characteristic equation
$\lambda _{1}\approx -7.7417,$ $\lambda _{2}\approx -0.2583,$
$\lambda _{3}\approx -1,$ $\lambda _{4}\approx -2,$ i.e., we
obtain the exponential stability of equation (3.11). We have zones
of parameters in which the exponential stability exists. Let us
set, for example, $\gamma _{1}=7,$ $\gamma _{2}=4,$ then the
Hurwitz criteria leads us to the following possible zone of the
exponential stability in the plane of the coefficients $\beta
_{1}$ and $\beta _{2}:$ \ $-410.6666<$ $\beta _{2}<0,$
$-1.75<\beta _{1}<-\beta _{2}-88+11\sqrt{-3\beta _{2}+64}.$

It is intuitively clear that $q \rightarrow \infty$ in the forms
of the kernels improves essentially our possibilities in feedback
delay stabilization by integral term. The form (3.2) will be
"optimal". Although in the modern real world application only the
case of $q=1,2$ are used, the studies of general case with $q =
\infty$ will be important problem.

\section{Results on Stability Based on analysis of auxiliary system}
For the second order equation
$$
x^{''}(t) + \gamma_{1}x^{'}(t) + \gamma_{2}x(t) = 0 , \eqno(4.1)\\\\
$$%
with parameters \(\gamma_{1},\gamma_{2}\), it is well known that
the necessary and sufficient condition of exponentially stability
is:
$$
\gamma_{1} > 0,\gamma_{2} > 0.\eqno(4.2)
$$%
Let us consider equation (4.1) when condition (4.2) is not
satisfied. In this situation, 3 important cases of stabilization
by integral feedback control can be considered.\\

\subsection{Case I: \(\gamma_{1} > 0,\gamma_{2} < 0.\)}

In this case, the form of the control can be as follows:
$$
(a)\ \ \ \ \beta\int_{0}^{t}e^{- \alpha(t - s)}x(s)ds,\ \ (b)\ \ \ \ \beta\int_{0}^{t}e^{- \alpha(t - s)}x^{'}(s)ds. \eqno(4.3)
$$%
Consider the situation where the control is of the type (a). Then
we come to the following equation:
$$
x^{''}(t) + \gamma_{1}x^{'}(t) + \gamma_{2}x(t) = \beta\int_{0}^{t}e^{- \alpha(t - s)}x(s)ds.
$$%
Then using Theorem 3.1 we obtain the following system of
ordinary differential equations:

\[
\left\{
\begin{array}{c}
x^{''} + \gamma_{1}x^{'} + \gamma_{2}x = y \\\\
y^{'} = -\alpha y + \beta x  \\
\end{array}%
.\right.\eqno(4.4)
\]
The characteristic equation of system (4.4) has the following
form:
$$
 \lambda^{3} + \left( \gamma_{1} + \alpha \right)\lambda^{2} + \left( \alpha + \gamma_{2} \right)\lambda + \alpha\gamma_{2} - \beta = 0.
$$%
The Hurwitz criteria gives necessary and sufficient conditions of the
exponential stability of (4.4):

$$
\left\{
\begin{array}{c}
 \gamma_{1} + \alpha > 0 \\
\alpha\gamma_{1} + \gamma_{2} > 0 \\
\alpha\gamma_{2} - \beta > 0 \\
\left( \gamma_{1} + \alpha \right)\left( \alpha\gamma_{1} + \gamma_{2} \right) + \beta - \alpha\gamma_{2} > 0 \\
\end{array}%
.\right. \eqno(4.5)
$$%

\begin{theorem}
     For any parameters \(\gamma_{1} > 0,\gamma_{2} < 0\ \)there
     exist
be parameters \(\alpha,\beta\ \)such that system (4.4) is
exponentially stable.
\end{theorem}

\begin{proof}
Solving system (4.5) we obtain that the following inequalities
will be satisfied
$$
\left\{
\begin{array}{c}
\alpha > - \frac{\gamma_{2}}{\gamma_{1}} \\\\
 - \gamma_{1}\alpha^{2} - \gamma_{1}^{2}\alpha - \gamma_{1}\gamma_{2} < \beta < \gamma_{2}\alpha \\
\end{array}%
\right.\eqno(4.6)
$$%
In the coordinate system $(\alpha,\beta)$, let us find the points
of intersection of graphs of functions
$\beta=-\gamma_{1}\alpha^{2} - \gamma_{1}^{2}\alpha -
\gamma_{1}\gamma_{2}$ and $ \beta = \gamma_{2}\alpha$. The result
will be the following 2 points: $(-\gamma_1,-\gamma_1\gamma_2)$,
$(-\frac{\gamma_2}{\gamma_1},-\frac{\gamma_{2}^2}{\gamma_1})$. The
stability zone of the solution of the system (4.4) can be seen in
the following figures
\begin{figure}[H]
\centering
\includegraphics[width=10cm]{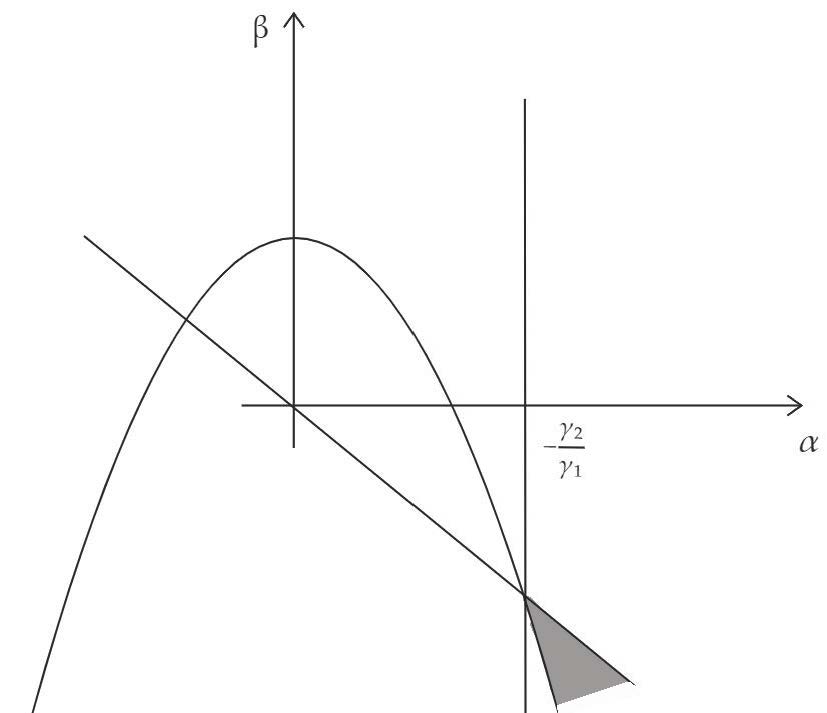}
\caption{Stability zone for system (4.4).}
\end{figure}

\end{proof}
\begin{Example}

Let \(\gamma_{2} = - 4,\ \gamma_{1} = 2\), then \(\alpha > 2\) and
let us choose \( \alpha = 5\), then \(- 62 < \beta < - 20\), and
we can set, for example, \(\beta = - 25\).

The characteristic polynomial has the form:

\[P_{3}(\lambda) = \lambda^{3} + 7\lambda^{2} + \lambda + 5.\]

Its roots are:

$$
\left\{
\begin{array}{c}
\lambda_{1} \approx - 1.73 \\
\lambda_{2} \approx - 0.02 + 0.52i. \\
\lambda_{3} \approx - 0.02 - 0.52i \\
\end{array}%
\right.\eqno(4.7) .
$$%
\end{Example}

Consider the situation where the control is of the type (b), i.e.,
we come to the following equation:
$$
x^{''}(t) + \gamma_{1}x^{'}(t) + \gamma_{2}x(t) = \beta\int_{0}^{t}e^{- \alpha(t - s)}x^{'}(s)ds.\eqno(4.8) \\
$$%
Then, using Theorem 3.1, we obtain the following system of
ordinary differential equations:

$$
\left\{
\begin{array}{c}
x^{''} + \gamma_{1}x^{'} + \gamma_{2}x = y \\\\
y^{'} = - \alpha y + \beta x^{'} \\
\end{array}%
\right.\eqno(4.9) .
$$%
The characteristic equation of the system (4.8) has the following form:

\[  \lambda^{3} + \left( \gamma_{1} + \alpha \right)\lambda^{2} + \left( \alpha\gamma_{1} + \gamma_{2} + \beta \right)\lambda + \alpha\gamma_{2} = 0.\]
The Hurwitz criteria gives necessary and sufficient conditions of the
exponential stability of (4.8):

$$
\left\{
\begin{array}{c}
\gamma_{1} + \alpha > 0 \\
\alpha\gamma_{1} + \gamma_{2} - \beta > 0 \\
\alpha\gamma_{2} > 0 \\
\left( \gamma_{1} + \alpha \right)\left( \alpha\gamma_{1} + \gamma_{2} - \beta \right) - \alpha\gamma_{2} > 0 \\
\end{array}%
\right.\eqno(4.10) .
$$%

\begin{theorem}
    If \(\alpha > 0\), \(\gamma_{1} > 0\), \(\gamma_{2} < 0\), then  system (4.8) is
unstable
\end{theorem}

\begin{proof}
If \(\gamma_{2} < 0\) and \(\alpha > 0\), then  \(\alpha\gamma_{2}
< 0\) and (4.10) cannot be fulfilled.
\end{proof}

\subsection{Case II: \(\gamma_{1} < 0,\gamma_{2} > 0.\)}

As in the previous case, the control can take the forms of  (4.3)
above. Consider the situation, where the control is of the type
(a).

\begin{theorem}
  System
(4.9) cannot be stable for \(\gamma_{1} < 0, \gamma_{2}
> 0\), \(\alpha > 0\).
\end{theorem}
\begin{Example}
 Let \(\gamma_{1} = - 4\), \(\gamma_{2} = 2\) and then the first two
inequalities in the system of inequalities (4.5) look like this

\[\left\{ \begin{matrix}
 - 4 + \alpha > 0 \\
 - 4\alpha + 2 > 0 \\
 \vdots \\
\end{matrix} \right.\ \]
As we can see, there is no solution to the system of inequalities.
\end{Example}

We have shown that not for all values \(\gamma_{i}\) with control
of the form (4.3) we can achieve the exponential stability.
Then we will look for a different type of control.\\
Consider the situation where the control is of the type (b). Then
we have equation (4.8).

Then, using Theorem 3.1, we obtain system (4.9) of ordinary
differential equations. The characteristic equation of the system
(4.9) has the following form:

\[ \lambda^{3} + \left( \gamma_{1} + \alpha \right)\lambda^{2} + \left( \alpha\gamma_{1} + \gamma_{2} - \beta \right)\lambda + \alpha\gamma_{2} = 0.\]
The Hurwitz criteria, gives necessary and sufficient conditions of
the exponential stability of (4.9) in the form of system of
inequalities (4.10).

\begin{theorem}
    For every fixed parameters \(\gamma_{1} < 0,\gamma_{2} > 0\ \) there exist parameters \(\alpha,\beta\ \)such that system (4.9) is
exponentially stable.
\end{theorem}

\begin{proof}
It is clear that \(\alpha\) and \(\beta\) can be chosen such that
the following inequalities to be satisfied :

\[\left\{ \begin{matrix}
\alpha > - \gamma_{1} \\
\beta < M \\
\end{matrix} \right.\ ,
\]
where $M = \frac{\gamma_{1}\alpha^{2} + \gamma_{1}^{2}\alpha +
\gamma_{1}\gamma_{2}}{\alpha + \gamma_{1}}$.

Indeed, the proof will be constructive, i. e., we will show how to
choose the values of parameters $\alpha$ and $\beta$.

1. At the first one we find conditions for the parameters
$\alpha$, then from the first inequality in the system (4.10) we
obtain $\alpha > -\gamma_1$. For the current case $\gamma_1 <0$
and $ \gamma_2 > 0$. So form these three inequalities we can
conclude that $\alpha> 0$, $\alpha \gamma_2 > 0$, and $\alpha
+\gamma_1>0$. As a result we get  ${{\alpha \gamma_2} \over
{\alpha +\gamma_1}}>0$.

2. After we have chosen the value of parameter $\alpha$ we need to
find the value of parameter $\beta$. Then from the remaining
inequalities in (4.10) we obtain:

$\beta < \gamma_1 \alpha +\gamma_2$

$\beta < \gamma_1 \alpha +\gamma_2- {{\alpha \gamma_2} \over
{\alpha +\gamma_1}}$

${{\alpha \gamma_2} \over {\alpha +\gamma_1}}>0$ then the last two
inequalities will be stratified if

$\beta < \gamma_1 \alpha +\gamma_2- {{\alpha \gamma_2} \over
{\alpha +\gamma_1}}= \frac{\gamma_{1}\alpha^{2} +
\gamma_{1}^{2}\alpha + \gamma_{1}\gamma_{2}}{\alpha +
\gamma_{1}}=M $

\end{proof}

\begin{Example}

Let \({\ \gamma_{1} = - 2,\ \ \gamma}_{2} = 4\), then \(\alpha >
2\). We can choose \(\alpha = 5\ \) and then \(\beta < -
12\frac{2}{3}\), we can set,  for example, \(\beta = - 15\).

The characteristic polynomial has the form:

\[P_{3}(\lambda) = \lambda^{3} + 3\lambda^{2} + 9\lambda + 20.\]

Its roots are:
\begin{equation*}
    \lambda_{1} \approx - 2.54, \lambda_{2} \approx - 0.22 + 2.79i, \lambda_{3} \approx - 0.22 - 2.79i.
\end{equation*}

\end{Example}

\subsection{Case III: \(\gamma_{1} < 0,\gamma_{2} < 0.\)}

In this case, the following control signals of the following types
are possible

\[(a)\ \ \ \beta_{1}\int_{0}^{t}e^{- \alpha_{1}(t - s)}x(s)ds + \beta_{2}\int_{0}^{t}e^{- \alpha_{2}(t - s)}x(s)ds,\]

\[(b)\ \ \ \beta_{1}\int_{0}^{t}e^{- \alpha_{1}(t - s)}x(s)ds + \beta_{2}\int_{0}^{t}e^{- \alpha_{2}(t - s)}x^{'}(s)ds,\]

\[(c)\ \ \ \beta_{1}\int_{0}^{t}e^{- \alpha_{1}(t - s)}x^{'}(s)ds + \beta_{2}\int_{0}^{t}e^{- \alpha_{2}(t - s)}x^{'}(s)ds,\]

In the case (a), we have the following equation:

\[x^{''}(t) + \gamma_{1}x^{'}(t) + \gamma_{2}x(t) = \beta_{1}\int_{0}^{t}e^{- \alpha_{1}(t - s)}x(s)ds + \beta_{2}\int_{0}^{t}e^{- \alpha_{2}(t - s)}x(s)ds.\]

Then using Theorem 3.1, we obtain the following system of
ordinary differential equations:

\[\left\{ \begin{matrix}
\begin{matrix}
x^{''} + \gamma_{1}x^{'} + \gamma_{2}x = y_{1} + y_{2} \\
\  \\
{y_{1}}^{'} = - \alpha_{1}y_{1} + \beta_{1}x \\
\  \\
{y_{2}}^{'} = - \alpha_{2}y_{2} + \beta_{2}x \\
\end{matrix}. \\
\  \\

\end{matrix} \right.\eqno(4.11) \]

The characteristic equation of system (4.11) has the following form:
\begin{gather*}
   \lambda^{4} + \left( \gamma_{1} + \alpha_{1} + \alpha_{2} \right)\lambda^{3} + \left( \alpha_{1}\gamma_{1} + \alpha_{2}\gamma_{1} + \alpha_{1}\alpha_{2} + \gamma_{2} \right)\lambda^{2} + \left( \alpha_{1}\alpha_{2}\gamma_{1} + \alpha_{1}\gamma_{2} + \alpha_{2}\gamma_{2} - \beta_{2} - \beta_{1} \right)\lambda +\\
   + \alpha_{1}\alpha_{2}\gamma_{2} - \alpha_{2}\beta_{1} -
   \alpha_{1}\beta_{2}=0.
\end{gather*}

The Hurwitz criteria, gives necessary and sufficient conditions of the
exponential stability of (4.11):

\[\left\{ \begin{matrix}
\gamma_{1} + \alpha_{1} + \alpha_{2} > 0 \\
\alpha_{1}\gamma_{1} + \alpha_{2}\gamma_{1} + \alpha_{1}\alpha_{2} + \gamma_{2} > 0 \\
\alpha_{1}\alpha_{2}\gamma_{1} + \alpha_{1}\gamma_{2} + \alpha_{2}\gamma_{2} - \beta_{2} - \beta_{1} > 0 \\
\alpha_{1}\alpha_{2}\gamma_{2} - \alpha_{2}\beta_{1} - \alpha_{1}\beta_{2} > 0 \\
\left( \gamma_{1} + \alpha_{1} + \alpha_{2} \right)\left( \alpha_{1}\gamma_{1} + \alpha_{2}\gamma_{1} + \alpha_{1}\alpha_{2} + \gamma_{2} \right)\left( \alpha_{1}\alpha_{2}\gamma_{1} + \alpha_{1}\gamma_{2} + \alpha_{2}\gamma_{2} - \beta_{2} - \beta_{1} \right) - \\
\left( \gamma_{1} + \alpha_{1} + \alpha_{2} \right)^{2}\left( \alpha_{1}\alpha_{2}\gamma_{2} - \alpha_{2}\beta_{1} - \alpha_{1}\beta_{2} \right) - \left( \alpha_{1}\alpha_{2}\gamma_{1} + \alpha_{1}\gamma_{2} + \alpha_{2}\gamma_{2} - \beta_{2} - \beta_{1} \right)^{2} > 0 \\
\end{matrix}. \right.\eqno(4.12) \]
To solve the system of inequalities (4.12), we will use the following idea. The first two inequalities in (4.12) are used to determine the values of $\alpha_1$ and $\alpha_2$. Then, these values are used to solve the remaining inequalities in (4.12) to obtain the values of $\beta_1$ and $\beta_1$.
Finding parameters \(\alpha_{1},\alpha_{2},\beta_{1},\beta_{2}\ \) will
be explained by 2 lemmas.

\begin{Lemma}
The following system of inequalities gives us some partial
solutions \(\alpha_{1}\) and \(\alpha_{2}\) of the first two
inequalities in system (4.12)
\[\left\{ \begin{matrix}
\alpha_{2} >  - \gamma_{1} \\
\  \\
\alpha_{1} > \frac{- \gamma_{1}\alpha_{2} - \gamma_{2}}{\gamma_{1} + \alpha_{2}} \\
\ \\
\alpha_{1} \neq \alpha_{2}\\
\end{matrix} .\right.\ \]
\end{Lemma}
\begin{proof}
    Figure 3 illustrates the solutions  $\alpha_1$ and $\alpha_2$ in the system of inequalities above.
    We can prove that these solutions are also solutions of the first two
inequalities in system (4.12). Indeed, from $\alpha_{2} >  -
\gamma_{1}$ we can conclude that $\alpha_{2} > 0$ and $
\gamma_{1}+\alpha_{2}>0$. So from these two inequalities and
$\alpha_{1} > \frac{- \gamma_{1}\alpha_{2} -
\gamma_{2}}{\gamma_{1} + \alpha_{2}}$ we can conclude that also
$\alpha_{1} > 0$. By summing $\alpha_{1} > 0$ and $\alpha_{2} >  -
\gamma_{1}$ we get that $ \alpha_{1} + \alpha_{2} > -\gamma_{1} $
and consequently $\gamma_{1} + \alpha_{1} + \alpha_{2} > 0$. This
is the first inequality in (4.12). Using $ \gamma_{1}
+\alpha_{2}>0$ and moving all term to the left side of the
inequality $\alpha_{1}
> \frac{- \gamma_{1}\alpha_{2} - \gamma_{2}}{\gamma_{1} +
\alpha_{2}} $ we get the second inequality in (4.12):
$\alpha_{1}\gamma_{1} + \alpha_{2}\gamma_{1} +
\alpha_{1}\alpha_{2} + \gamma_{2} > 0 $. So the solution
$\alpha_1$ and $\alpha_2$ is also solution of the first two
inequalities in the system (4.12).

\end{proof}

\begin{figure}[H]
\centering
\includegraphics[ width=10cm]{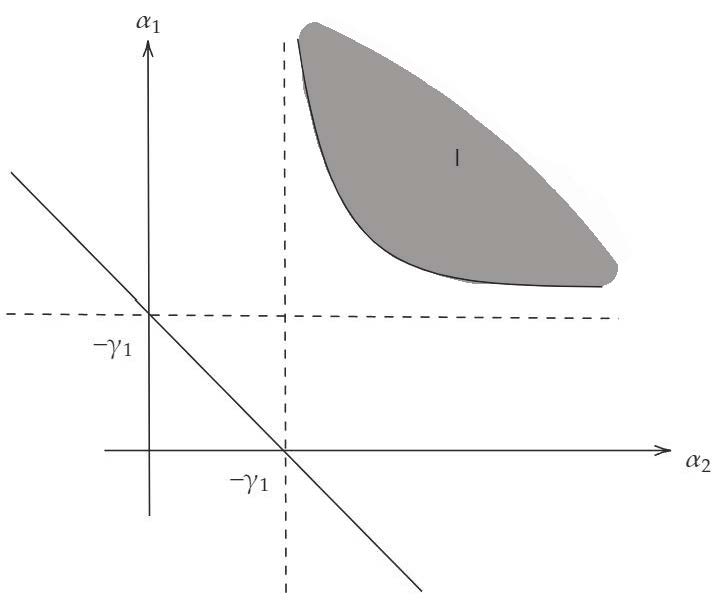}
\caption{Partial solution of the first two inequalities in (4.12)
}
\end{figure}

\begin{Lemma}
 In order to find the values of \(\beta_{1},\beta_{2}\) you use
the following transformation

\[\binom{\beta_{1}}{\beta_{2}} = \frac{1}{\alpha_{2} - \alpha_{1}}\begin{pmatrix}
 - \alpha_{1} & \alpha_{2} \\
1 & - 1 \\
\end{pmatrix}\binom{C - u_{1}}{D - v_{1}},\eqno(4.13)\]
where \(u_{1},v_{1}\) are chosen from the zone in which the
following inequalities are fulfilled

\[\left\{ \begin{matrix}
u_{1} > 0 \\
0 < v_{1} < - \frac{1}{A^{2}}u_{1}^{2} + \frac{B}{A}u_{1} \\
\end{matrix}, \right.\ \eqno(4.14) \]
where
\(A = \gamma_{1} + \alpha_{1} + \alpha_{2}\ ,\ B =
\alpha_{1}\gamma_{1} + \alpha_{2}\gamma_{1} + \alpha_{1}\alpha_{2}
+ \gamma_{2},\ C = \alpha_{1}\alpha_{2}\gamma_{1} +
\alpha_{1}\gamma_{2} + \alpha_{2}\gamma_{2}\ ,\ D =
\alpha_{1}\alpha_{2}\gamma_{2}.\)
\end{Lemma}

\begin{proof}
  Use the following linear reversible transformation to introduce new variables
\begin{equation*}
  \binom{u_{1}}{v_{1}} =\begin{pmatrix}
      -1& -1\\
      -\alpha_2&-\alpha_1
  \end{pmatrix}
  \binom{\beta_{1}}{\beta_{2}}+\binom{C}{D}\eqno(4.15)
\end{equation*}
After replacing the variables, the system of remaining
inequalities (4.12) can be written in the following form:
\begin{equation*}
 \begin{cases}
  u_1>0,v_1>0  \\

A\cdot B u_1-A^2v_1-u_1^2>0 \\
 \end{cases}
\end{equation*}
The last system of inequalities can be presented in form (4.14).
Figure 4 illustrates this solution.

\begin{figure}[H]
\centering
\includegraphics[ width=10cm]{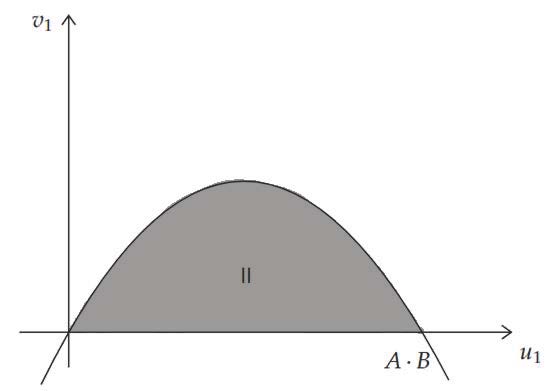}
\caption{Stability zone for parameters $u_1$and $v_1$. }
\end{figure}
\end{proof}
\begin{Example}
   Let \({\ \gamma_{1} = - 2,\ \ \gamma}_{2} = - 5,\) area of the parameter
\(\alpha_{2} > 2\) . Solving the system of equations and find the
parameters, we obtain \(\alpha_{1} \approx 3.55\ ,\ \alpha_{2} \approx 8.45\).
Using the second lemma we obtain
\(v_{1} < - \frac{1}{100}u_{1}^{2} + \frac{1}{10}u_{1}\) , Then
we\textquotesingle ll choose \(u_{1} = 2\) and then \(v_{1} < 0.16\) ,
then we choose \(v_{1} = 0.1\) and calculate the parameters
\(\beta_{1},\ \beta_{2}\) by using transformation (4.13). After
substituting all parameters, we obtain the following characteristic
equation
\(\lambda^{4} + 10\lambda^{3} + \lambda^{2} + 2\lambda + 0.1 = 0\) which
roots are
\(\lambda_{1} \approx - 9.91,\ \lambda_{2} \approx - 0.05,\ \lambda_{3} \approx - 0.14 - 0.44i,\ \lambda_{4} \approx - 0.14 + 0.44i\).
\end{Example}
Let\textquotesingle s look at the last case.\\
In this case we obtain the following equation:

\[x^{''}(t) + \gamma_{1}x^{'}(t) + \gamma_{2}x(t) = \beta_{1}\int_{0}^{t}e^{- \alpha_{1}(t - s)}x{'}(s)ds + \beta_{2}\int_{0}^{t}e^{- \alpha_{2}(t - s)}x{'}(s)ds.\]
Then using Theorem 3.1, we obtain the following system of
ordinary differential equations:

\[\left\{ \begin{matrix}

x^{''} + \gamma_{1}x^{'} + \gamma_{2}x = y_{1} + y_{2} \\
\  \\
y_{1}^{'} = - \alpha_{1}y_{1} + \beta_{1}x^{'} \\
\  \\
y_{2}^{'} = - \alpha_{2}y_{2} + \beta_{2}x^{'} \\
\  \\
\end{matrix} \right. .\eqno(4.15) \]

The characteristic equation of the system (4.13) has the following form:

\[ \lambda^{4} + \left( \gamma_{1} + \alpha_{1} + \alpha_{2} \right)\lambda^{3} + \left( \alpha_{1}\gamma_{1} + \alpha_{2}\gamma_{1} + \alpha_{1}\alpha_{2} + \gamma_{2} - \beta_{1} - \beta_{2} \right)\lambda^{2}+ \]\

\begin{equation*}
    + \left( \alpha_{1}\alpha_{2}\gamma_{1} + \alpha_{1}\gamma_{2} + \alpha_{2}\gamma_{2} - {\alpha_{1}\beta}_{2} - \alpha_{2}\beta_{1} \right)\lambda +
    \alpha_{1}\alpha_{2}\gamma_{1}=0.
\end{equation*}

The necessary condition of exponential stability of the solution
of system (4.15) is the condition of positivity of coefficients of
the characteristic equation, which in our case is not fulfilled,
since \(\alpha_{1}\alpha_{2}\gamma_{1} < 0\).

\section{Application to stabilization of drone's flight}

Let us use the method described in Section 3 for angular
stabilization of a drone by autopilot.

Let us define controlling term in integral form  in eq. (2.3):
$$
\begin{array}{c}
-n_B \delta_B (t)= {\frac{\beta _{1}
}{n_{22}}}\theta(t)+\beta_2\int_0^t e^{-\alpha_2 (t-s)}
\alpha(s)ds = \\
{\frac{\beta _{1} }{n_{22}}}(\theta(0)+ \int_0^t
{\theta^{\prime}(s)}ds)+\beta_2\int_0^t e^{-\alpha_2 (t-s)}
\alpha(s)ds =\\
{\frac{\beta _{1} }{n_{22}}}(\theta(0)+ \int_0^t
{(n_{22}\alpha(s)-f_{2}(s))}ds)+\beta_2\int_0^t e^{-\alpha_2
(t-s)}
\alpha(s)ds =\\
{\frac{\beta _{1} }{n_{22}}}\theta(0)+ \beta _{1}\int_0^t
{\alpha(s)}ds+ \int_0^t (-{{\frac{\beta _{1}
}{n_{22}}}f_{2}(s))}ds+\beta_2\int_0^t e^{-\alpha_2 (t-s)}
\alpha(s)ds =\\
\beta_1 \int_0^t x(s)ds+\beta_2 \int_0^t e^{-\alpha_2
(t-s)} x(s)ds+\int_0^t r_1 (s)ds+y_1 (0)%
\end{array}%
,\eqno(5.1)
$$
where $x=\alpha (t)$, $y_{1}={\frac{\beta _{1}\theta }{n_{22}}}$, $r_{1}(t)=-%
{\frac{\beta _{1}f_{2}(t)}{n_{22}}}$.

The equations (2.3) in this case can be written in the form:
$$
\left\{
\begin{array}{c}
y_1(t)=y_1 (0)+\beta_1 \int_0^t x(s)ds+ \int_0^t r_1 (s)ds \\
x^{\prime\prime}(t)+\gamma_{1}x ^{\prime}(t)+\gamma_{2}x(t)= \\
\beta_1 \int_0^t x(s)ds+\beta_2 \int_0^t e^{-\alpha_2 (t-s)}
x(s)ds+(r_2(t)+\int_0^t r_1 (s)ds+y_1 (0))%
\end{array}%
\right. ,\eqno(5.2)
$$
where $r_{2}(t)=f_{3}(t)+f_{2}^{\prime }(t)+n_{33}f_{2}(t)$.

The full system (5.2) is equivalent to the following system of
ordinary differential equations by introducing new variable
$y_{2}(t)$ with initial condition $y_{2}(0)=0$ (see Theorem 3.1):
$$
\left\{
\begin{array}{c}
y_2^{\prime}(t)=\beta_2 x(t) -\alpha_2 y_2(t) \\
y_1^{\prime}(t)=\beta_1 x(t) + r_1(t) \\
x^{\prime\prime}(t)+\gamma_{1}x ^{\prime}(t)+\gamma_{2}x(t)=
y_1(t) +
y_2(t)+r_2(t)%
\end{array}%
\right. ,\eqno(5.3)
$$
where $\delta _{B}(t)=-{\frac{y_{1}+y_{2}}{n_{B}}}$, $\gamma
_{1}>0$ is a known positive constant, $\gamma _{2}$ is a known
constant with arbitrary sign, $\beta _{1}$, $\beta _{2}$, $\alpha
_{2}$ are arbitrary constants.

Thus, we have to study stability of (5.3) with respect to $r_{1}(t)$, $%
r_{2}(t)$ to be sure that "small" errors $r_{1}(t)$, $r_{2}(t)$ do
not imply "big" perturbations in the solution. According to the
classical Bohl-Perron theorem \cite{AMR,AP} for this case, the
homogeneous equations:
$$
\left\{
\begin{array}{c}
y_2^{\prime}(t)=\beta_2 x(t) -\alpha_2 y_2(t) \\
y_1^{\prime}(t)=\beta_1 x(t) \\
x^{\prime\prime}(t)+\gamma_{1}x ^{\prime}(t)+\gamma_{2}x(t)= y_1(t) + y_2(t)%
\end{array}%
\right. .\eqno(5.4)
$$
has to be exponentially stable.

The Hurwitz criteria, gives necessary and sufficient conditions of
the exponential stability of (5.4):

$$
\left\{
\begin{array}{c}
\gamma_1+\alpha_2>0 \\
\gamma_1 \alpha_2+\gamma_2>0 \\
\alpha_2 \gamma_2-(\beta_1+\beta_2 )>0 \\
-\alpha_2 \beta_1>0 \\
(\gamma_1+\alpha_2 ) (\gamma_1 \alpha_2+\gamma_2 )(\alpha_2
\gamma_2-(\beta_1+\beta_2 ))-(\gamma_1+\alpha_2 )^2 (-\alpha_2 \beta_1 )- \\
-(\alpha_2 \gamma_2-(\beta_1+\beta_2 ))^2>0%
\end{array}%
\right. .\eqno(5.5)
$$

This system has solution if and only if the following system has
solution for $\gamma_1>0$:

$$
\left\{
\begin{array}{c}
\alpha_2>max(-\gamma_1,-{\frac{\gamma_2 }{\gamma_1}});\alpha_2\neq 0 \\
\alpha_2 \gamma_2>B>-\gamma_1(\gamma_1 \alpha_2+\gamma_2+\alpha_2^2) \\
{\frac{(\gamma_1(\gamma_1
\alpha_2+\gamma_2+\alpha_2^2)+B)(\alpha_2\gamma_2-B) }{(\gamma_1+\alpha_2)^2}%
}>A>0 \\
\beta_1=-{\frac{A }{\alpha_2}} \\
\beta_2=B-\beta_1%
\end{array}%
\right. ,\eqno(5.6)
$$
where $A$ and $B$ are corresponding constants.

 If $\alpha_2>max(-\gamma_1,-{\frac{\gamma_2 }{\gamma_1}})$,
then $\alpha_2 \gamma_2>-\gamma_1(\gamma_1
\alpha_2+\gamma_2+\alpha_2^2)$.

If $\alpha_2 \gamma_2>B>-\gamma_1(\gamma_1
\alpha_2+\gamma_2+\alpha_2^2)$, then $(\gamma_1(\gamma_1
\alpha_2+\gamma_2+\alpha_2^2)+B)(\alpha_2\gamma_2-B)>0$.

As a result, the system (5.5) always has a non-empty set of
solutions described in (5.6).

{\bf Example 5.1.}

Let us choose: $M=1.6$, $H=12 km$, $\tau_a=4 s$ $n_0=0.68$; $n_{33}=2.42$; $%
n_{22}=2.4$; $n_{32}=-36$; $n_B=46$. So, we can get: $%
\gamma_1=n_0+n_{33}+n_{22}=5.5$
$\gamma_2=n_{32}+n_{33}n_{22}=-30.192$. From (5.6) we can get:

$\alpha _{2}>max(-\gamma _{1}=-5.5,-{\frac{\gamma _{2}}{\gamma
_{1}}}\approx 5.49)\approx 5.49; \alpha _{2}\neq 0,$ then we can
choose $\alpha _{2}=6$.

$\alpha _{2}\gamma _{2}=-181,152>B>-\gamma _{1}(\gamma _{1}\alpha
_{2}+\gamma _{2}+\alpha _{2}^{2})=-213.444,$ then we can choose
$B=-200$.

${\frac{(\gamma _{1}(\gamma _{1}\alpha _{2}+\gamma _{2}+\alpha
_{2}^{2})+B)(\alpha _{2}\gamma _{2}-B)}{(\gamma _{1}+\alpha _{2})^{2}}}={%
\frac{(13.444\ast 18.848)}{132.25}}=1.92>A>0,$ then we can choose $A=1$, $%
\beta _{1}=-{\frac{A}{\alpha _{2}}}\approx -0.167,$ $\beta
_{2}=B-\beta _{1}=-199,833.$

From (5.1) we get:

\[
\begin{array}{c}
\delta_B (t)=-{\frac{\beta_1 \int_0^t x(s)ds+\beta_2 \int_0^t
e^{-\alpha_2 (t-s)} x(s)ds +y_{1}(0)}{n_B}}={\frac{0.167 \int_0^t
x(s)ds+199,833 \int_0^t e^{-6
(t-s)} x(s)ds +y_{1}(0)}{46}}%
\end{array}%
\]

$\theta ={\frac{n_{22}y_{1}}{\beta _{1}}}=-1.43y_{1},$ $\alpha =x$, where $x$%
, $y_{1}$ are solutions of (5.4).

We verified exponential stability of some solution for (5.4) by numerical simulation using  Matlab programm dde23 for the found parameters.
The result can be seen on Figure 5.

\begin{figure}[H]
\centering
\includegraphics[ width=10cm]{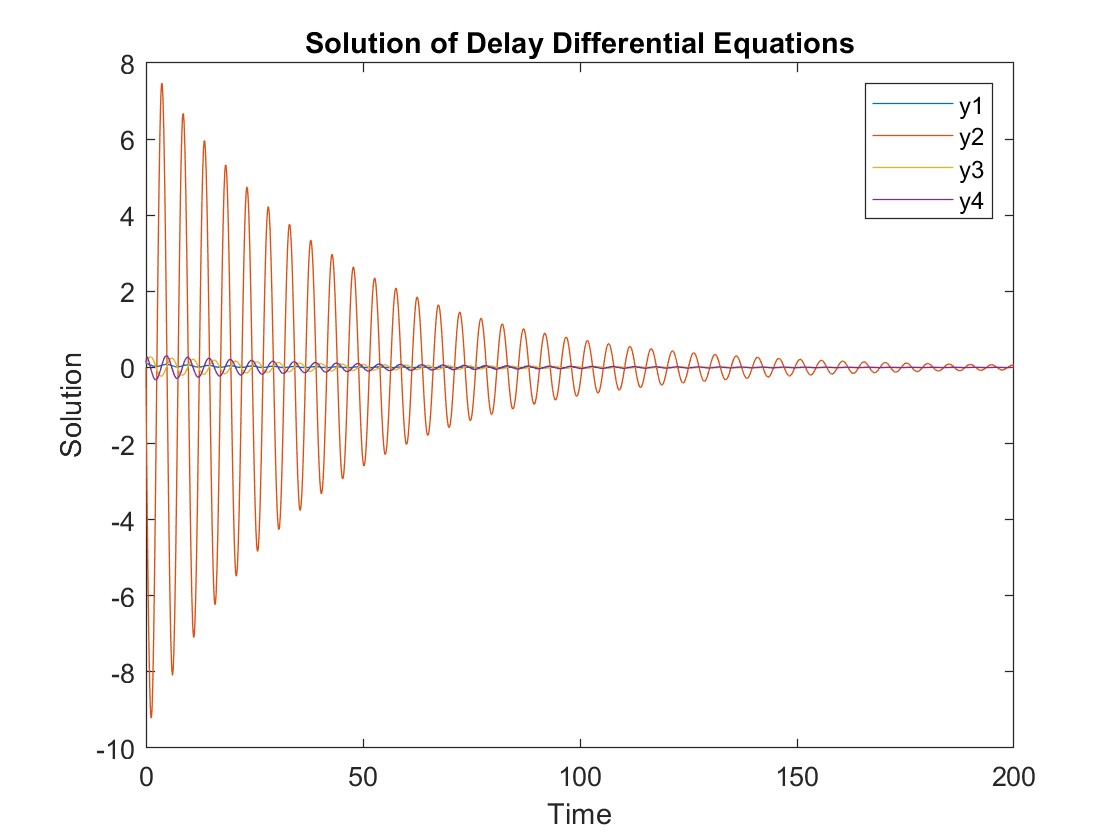}
\caption{Verification of exponential stability of the system (5.4) for the found parameters. Here $y1=y_{1}$,$y2=y_{2}$,$y3=x$,$y4={\frac{dx}{dt}}$}
\end{figure}

For equations (2.3) we found control parameters to obtain
exponentially stable trajectory. We can obtain solution (2.3)
through solutions of (5.3).

\bigskip

\section*{Disclosure statement}
No potential conflict of interest was reported by the authors.


\begin{thebibliography}{99}

\bibitem{Abahre1} A. Abahre, D. Elbhar, O. Kupervasser, H. Kutomanov, G. Rubin,
G. Shimoni, R. Yavich, Fusion of inertial navigation system with
10 axes and computer vision for UAV navigation, Functional
Differential Equations, V. 26, No 3-4, 2020, p. 141-146.

\bibitem{Abdillah} Moussa Abdillah, El Mehdi Mellouli , Touria Haidi, A new intelligent controller
based on integral sliding mode control and extended state observer for nonlinear
MIMO drone quadrotor, International Journal of Intelligent Networks, Volume 5,
2024, 49-62.

\bibitem{AgarwalLup} R.P. Agarwal, V. Lupulescu, D. O'Regan and A. Younus,
Floquet theory for a Volterra integro-dynamic system, Applicable Analysis,
2014, http://dx.doi.org/10.1080/00036811.2013.867019

\bibitem{Aguiar2008} M.Aguiar, B.Kooi, N.Stollenwerk, Epidemiology of dengue
fever: A model with temporary cross immunity and possibly
secondary infection shows bifurcations and chaotic behaviors in
wide parameter region, Math. Model. Nat. Phenom., 3 (2008), 48-70.

\bibitem{Avasker1} Avasker, S., Domoshnitsky, A., Kogan, M., Kupervasser,
O., Kutomanov, H., Rofsov, Y., Volinsky, I., Yavich, R. (2020). A method for
stabilization of drone flight controlled by autopilot with time delay. SN
Applied Sciences, 2(2), Article 225.
https://doi.org/10.1007/s42452-020-1962-6.

\bibitem{AMR} N.V.Azbelev, V.P.Maksimov and L.F.Rakhmatullina. {\it %
Introduction to theory of functional-differential equations,} Nauka, Moscow,
1991.

\bibitem{AP} Azbelev, N.\thinspace V., and Simonov, P.\thinspace M. :
Stability of Differential Equations with Aftereffect. {\em Stability and
Control: Theory, Methods and Applications}, {\bf 20}. Taylor \& Francis,
London, 2003.

\bibitem{Becker} Becker LC, Burton TA, Krisztin T. Floquet theory for a
Volterra equation. J. London Math. Soc. 1988; 37--2:141--147.

\bibitem{Belbas} S.A.Belbas, Floquet theory for integral and
integro-differential equations. Classical Analysis and ODEs (math.CA)
arXiv:1205.5302v1 [math.CA].

\bibitem{Bellman} R.Bellman,Methods of nonlinear analysis, Academic Press,
New York and London, 1973.

\bibitem{BDK} L.Berezansky, A.Domoshnitsky, R.Koplatadze, Oscillation,
Stability and Asymptotic Properties for second and Higher Order Functional
Differential Equations, CRC Press, Taylor \& Francis Group, 2020.

\bibitem{Bodner1} Bodner V.A, Kozlov M.S. Stabilization of aerial vehicles
and autopilots. OboronGIz, Moscow , 1961 (in Russian).

\bibitem{Burton} T.A.Burton. Volterra Integral and Differential Equations,
N.Y., Academic Press, 1983.

\bibitem{Cacace} Cacace, F., Germani, A., \& Manes, C. Nonlinear systems
with multiple timevarying measurement delays. SIAM Journal on
Control and Optimization, 52(3), 2014, 1862--1885.

\bibitem{Chen} D.Chen, Z.Xu, Global dynamics of a delayed diffusive
two-strain disease model, Diff. Eq. Appl., 8 (2016), 99-122.

\bibitem{cord2} C.Corduneanu. Integral Equations and Stability Feedback
Systems.-New York, London: Academic Press, 1973.

\bibitem{Corduneanu} C.Corduneanu. Integral Equations and Applications,
Campbridge University Press, Campbridge,1991.

\bibitem{DomGolInfinite} A.Domoshnitsky and Ya.Goltser, Positivity of
solutions to boundary value problems for infinite functional differential
systems, Mathematical and Computer Modelling. Vol. 45, No. 11-12, June 2007,
1395-1404.

\bibitem{DomKup2017} A.Domoshnitsky, O.Kupervasser, \textquotedblleft
Vision-based drone navigation using ground relief (DTM)\textquotedblright ,
Hangzhou Avisi Electronics Co., Ltd. -- China private company, 2017-2018.

\bibitem{DomKup2018} A.Domoshnitsky, O.Kupervasser, \textquotedblleft
Vision-based navigation of ground robots from upper
position\textquotedblright , KAMIN -- Israel Innovation Authority, 2018-2020.

\bibitem{DomKup2020} A.Domoshnitsky, O.Kupervasser, \textquotedblleft
Vision-based drone navigator\textquotedblright , NOFAR -- Israel Innovation
Authority, SWYFT Aeronautical Technology, Ltd.-- private company, 2020-2023.

\bibitem{Driessche} P.Van den Driessche. Some epidemiological models with
delays. In:\ Differential Equations and Applications to biologt and to
Industry, World Sci. (1996), 507-520.

\bibitem{Drozd1} A.D.Drozdov and V.B.Kolmanovskii. Stability in
Viscoelasticity, North Holand, Amsterdam, 1994.

\bibitem{drozd2} A.D.Drozdov. Stability of integro-differential equation
with periodic operator coefficients, Quart.J. Mech. Appl. Math., 49,235-260
(1996).

\bibitem{drozd} A.D.Drozdov. Explicit stability conditions for
integro-differential equations with periodic coefficients, Mathematical
Methods in the Applied Sciences, vol.21, 565-588 (1998).

\bibitem{Gil} A.D.Drozdov and M.I.Gil. Stability of linear
integro-differential equation with periodic coefficients, Quart. Appl.
Math., 54,609-624 (1996).

\bibitem{Guan} J.Guan, L.Wu, M.Chen, X.Dong, H.Tang, Z.Chen, The stability
and Hopf bifurcation of the dengue fever model with time delay, It. J. Pure
App. Math, 37 (1973), 139--156.

\bibitem{Hara} T. Hara and R. Miyazaki, Equivalent condition for stability
of Volterra integro--differential equations, J. Math. Anal. Appl.,
174(1993), 298--316.

\bibitem{Huang} C.Huang, J.Cao, F. Wen, X,Yang at al. Stability analysis of
SIR model with distributed delay on complex networks, PloS One, 11 (2016),
e0158813. https://doi.org/10.1371/ jpurnal.pone.0158813.

\bibitem{Fabrizio} M.Fabrizio and A.Morro. Mathematical Problems in Linear
Viscoelasticity, SIAM Stud.Appl. Math.,Philadelphia,1992.

\bibitem{Ferry} J.D.Ferry.Viscoelastic properties of polymers. John Wiley
and Sons Inc.,1970.

\bibitem{Filatov} A.N.Filatov. The Averaging Method in Differential and
Integro-Differential Equations, Tashkent, FAN, USSR, 1971 (in Russian).

\bibitem{Fridman1} Fridman E. A refined input delay approach to sampled-data
control, Automatica, 46, (2014), 421--427.

\bibitem{Fridman2} Fridman E. Introduction to Time-Delay Systems: Analysis
and Control, Birkhauser, 2014.

\bibitem{Golden} J.M.Golden and G.A.C.Graham. Boundary value problems in
linear viscoelasticity. Springer-Verlag,1988.

\bibitem{Gripenberg} G.Gripenberg, S.-O.Londen and O.Staffans. Volterra
Integral and Functional Equations, Campbridge University Press, Campbridge,
1990.

\bibitem{Gurtin} M.E.Gurtin and Pipkin, A General Theory of Heat conduction
with FiniteWave Speeds, Arch. Rat. Mech. Anal. 31 (1968) 113-126.

\bibitem{Indicator1} Airspeed Indicator,
https://www.cfinotebook.net/notebook/avionics-and-instruments/airspeed-indicator

\bibitem{Kolmogorov} A.N.Kolmogorov and S.V.Fomin, Elements of function
theory and functional analysis, Nauka, Moscow, 1972 (in Russian).

\bibitem{Kotaro1} Kotaro Haneda, Kenei Matsudaira, Ryusuke Noda, Toshiyuki
Nakata, Satoshi Suzuki, Hao Liu and Hidetoshi Takahashi, "Compact
Sphere-Shaped Airflow Vector Sensor Based on MEMS Differential
Pressure Sensors", Sensors 2022, 22(3), 1087
https://www.mdpi.com/1424-8220/22/3/1087

\bibitem{Lena1} E.Litsyn, On the formula for general solution of infinite
system of functional differential equations, Functional Differential
Equations, v.2, 1994, 111-121.

\bibitem{Lena2} E.Litsyn, On the general theory of linear functional
differential equations, Differentsial'nye Uravnenia, 1988, vol.24, pp.
977-986.

\bibitem{Madgwick1} Madgwick filter discription
https://habr.com/ru/articles/255661/ (in Russian)

\bibitem{Madgwick2} S. Madgwick, An Efficient Orientation Filter for
Inertial and Inertial/Magnetic Sensor Arrays; Technical Report;
Report x-io and University of Bristol: Bristol, UK, 30 April 2010.
https://www.samba.org/tridge/UAV/madgwick\_internal\_report.pdf


\bibitem{Novick} A.Novick-Cohen. Conserved Phase-Field Equations with
Memory, in \textquotedblleft Curvature Flows and Related
Topics,\textquotedblright\ A.Damlamian, J.Spruck and A.Visintin, eds., pp.
179-197, GAKUTO Internat.Ser. Math. Sci. Appl., 5. Gakkotosho, Tokyo,1995.

\bibitem{Pers1} K.P.Persidskii, On stability of solutions to a countable
systems of differential equations, Izv. AS Kaz. SSR, Math.and Mech.,2, 1948,
2-35 (in Russian).

\bibitem{Pers2} K.P.Persidskii, Infinite systems of differential equations,
Izv. AS Kaz. SSR, Math.and Mech. 4(8), 1956, 3-11 (in Russian).

\bibitem{Pers3} K.P.Persidskii, Countable systems of differential equations
and stability of their solutions, Izv. AS Kaz. SSR, Math.and Mech. 7(11),
1959, 52-71 (in Russian).

\bibitem{PSM} A. Ponosov, A. Shindiapin and J. Miguel, The W--transform
links delay and ordinary differential equations, Functional Differential
Equations, 9(2002), 437--470.

\bibitem{Pruss} J. Pruss, {\it Evolutionary Integral Equations and
Applications}, Monogr. Math., 87, Birkhauser, Basel, 1993.

\bibitem{Reich} N.G.Reich, S. Shrestha, A.A.King, at al., Ineraction betwee
serotypes of dengue hight epidemiological impact of cross-immunity, J.R.Soc.
Interface, 10 (2013), 4-14. https://doi.org/10.1098/rsif.2013.0414.

\bibitem{Renardy} M.Renardy, W.J.Hrusa and J.A.Nohel. Mathematical Problems
in Viscoelasticity, Longman, New York, 1987.

\bibitem{Rotstein} H.G.Rotstein, S.Brandon, A.Novick-Cohen and
A.Nepomnyashchy. Phase field equations with memory: the hyperbolic case.
SIAM J. Appl. Math., vol. 62, N 1, pp.264-282.

\bibitem{Shilo} G.E. Shilov, {\it Mathematical Analysis, Specialized Course}%
, FizMatGiz, 1965 (in Russian).

\bibitem{Song} Y. Song, L. He, D. Zhang, J. Qian, and J. Fu, “Neuroadaptive fault tolerant control
of quadrotor UAVs: A more affordable solution,” IEEE Trans. Neural Netw. Learn.
Syst., vol. 30, no. 7, pp. 1975–1983, Jul. 2019.

\bibitem{Triputra} F. R. Triputra, B. R. Trilaksono, T. Adiono, R. A. Sasongko, M. Dahsyat, Nonlinear
Dynamic Modeling of a Fixed-Wing Unmanned Aerial Vehicle: a Case Study of
Wulung. Journal of Mechatronics, Electrical Power, and Vehicular Technology.
https://doi.org/10.14203/j.mev.2015.v6.19-30

\bibitem{Valeev} K.G.Valeev and O.A.Zhautykov, Infinite Systems of
Differential Equations, Nauka, Alma-Ata, 1974 (in Russian).

\bibitem{Vanessa} Vanessa Steindorf, Sergio Oliva and Jianhong Wu, Cross
immunity protection and antibody-dependent enhancement in a distributed
delay dynamic model, Mathematical Biosciences and Engineering, MBE 19(3):
2950-2984, DOI 10.3934/mbe.2022136

\bibitem{Xu2017} J.Xu, Y.Geng, Y.Zhou, Global attractivity of a multi-group
model with distrubuted delay and vaccination, Math. Meth. Appl. Sci,, 40
(2017), 1475-1486. https://doi.org/10.1002/mma.4068

\bibitem{Yang} S. Yang and B. Xian, “Energy-based nonlinear adaptive control design for the quadrotor
UAV system with a suspended payload,” IEEE Trans. Ind. Electron., vol. 67, no.
3, pp. 2054–2064, Mar. 2020.

\bibitem{Yajing} Yajing Yu , Jian Guo, Choon Ki Ahn and Zhengrong Xiang, Neural Adaptive Distributed
Formation Control of Nonlinear Multi-UAVs With Unmodeled Dynamics,
IEEE TRANSACTIONS ON NEURAL NETWORKS AND LEARNING SYSTEMS,
VOL. 34, NO. 11, NOVEMBER 2023 9555.

\end{thebibliography}
\end{document}